\documentclass[11pt]{article}

\usepackage[preprint]{acl}

\usepackage{times}
\usepackage{latexsym}
\usepackage{amsmath}
\usepackage{amssymb}
\usepackage{booktabs}
\usepackage{multirow}

\usepackage{amsthm}
\usepackage{mathtools}
\usepackage{thmtools}
\usepackage{hyperref}
\usepackage{natbib}
\usepackage{enumitem}
\usepackage{xcolor}
\usepackage{algorithm}
\usepackage{algpseudocode}

\newtheorem{theorem}{Theorem}
\newtheorem{proposition}[theorem]{Proposition}
\newtheorem{corollary}[theorem]{Corollary}
\newtheorem{lemma}[theorem]{Lemma}
\newtheorem{definition}[theorem]{Definition}

\newtheorem{remark}{Remark}

\newcommand{\cY}{\mathcal{Y}}
\newcommand{\cA}{\mathcal{A}}
\newcommand{\R}{\mathbb{R}}

\newcommand{\supp}{\mathrm{supp}}

\newcommand{\proj}{\Pi}
\newcommand{\ip}[2]{\langle #1, #2 \rangle}

\newcommand{\ind}{\mathbb{1}}

\usepackage[T1]{fontenc}

\usepackage[utf8]{inputenc}
\usepackage[autostyle=false, style=english]{csquotes}
\MakeOuterQuote{"}

\usepackage{microtype}

\usepackage{inconsolata}

\usepackage{graphicx}

\title{Knowledge Divergence and the Value of Debate for Scalable Oversight}

\author{Robin Young \\
  Department of Computer Science and Technology \\
  University of Cambridge \\
  Cambridge, UK \\
  \texttt{robin.young@cl.cam.ac.uk}}

\begin{document}
\maketitle

\begin{abstract}
AI safety via debate and reinforcement learning from AI feedback (RLAIF) are both proposed methods for scalable oversight of advanced AI systems, yet no formal framework relates them or characterizes when debate offers an advantage. We analyze this by parameterizing debate's value through the geometry of knowledge divergence between debating models. Using principal angles between models' representation subspaces, we prove that the debate advantage admits an exact closed form. When models share identical training corpora, debate reduces to RLAIF-like where a single-agent method recovers the same optimum. When models possess divergent knowledge, debate advantage scales with a phase transition from quadratic regime (debate offers negligible benefit) to linear regime (debate is essential). We classify three regimes of knowledge divergence (shared, one-sided, and compositional) and provide existence results showing that debate can achieve outcomes inaccessible to either model alone, alongside a negative result showing that sufficiently strong adversarial incentives cause coordination failure in the compositional regime, with a sharp threshold separating effective from ineffective debate. We offer the first formal connection between debate and RLAIF, a geometric foundation for understanding when adversarial oversight protocols are justified, and connection to the problem of eliciting latent knowledge across models with complementary information.
\end{abstract}

\section{Introduction}
\label{sec:intro}

Scalable oversight, namely the problem of supervising AI systems on tasks too complex for direct human evaluation, has motivated two prominent families of approaches. AI safety via debate \citep{irving2018aisafetydebate, brown-cohen2024scalable} proposes pitting two AI models against each other in a structured argument, with a human judge evaluating the transcript. Reinforcement learning from AI feedback (RLAIF), as implemented in Constitutional AI \citep{bai2022constitutional}, trains models to self-critique against a set of principles. Both approaches aim to amplify a weak overseer's ability to evaluate complex model outputs.

Despite their shared motivation, these methods have developed in isolation. The debate literature, rooted in computational complexity theory, studies the power of interactive proof systems with polynomial-time verifiers \citep{brown-cohen2024scalable,browncohen2025proverestimatordebate}. The RLAIF literature, rooted in preference learning, studies how constitutional principles shape model distributions. No formal framework relates these approaches or characterizes when one offers an advantage over the other.

We analyze this question. Under our assumptions, our central observation is that the value of debate depends on the knowledge divergence between the debating models, which is a quantity that existing debate theory, which treats provers as abstract computational agents, does not formalize. Using principal angles between models' representation subspaces, we derive a bound on the debate advantage as the improvement in constitutional score achievable through adversarial multi-model interaction over single-model optimization.

Our contributions are as follows. First, we prove that the debate advantage admits the exact closed form $\Delta = \sqrt{(K_A^*)^2 + \eta^2} - K_A^*$ and satisfies tight bounds governed by a private information value $\eta$ derived from the principal angle spectrum between models' representation subspaces (Theorem~\ref{thm:main-bound}). Same-corpus equivalence falls out immediately: when models share training data, $\eta = 0$ and debate reduces to RLAIF (Corollary~\ref{cor:same-corpus}). Then, we classify three regimes of knowledge divergence (shared, one-sided, and compositional) with existence results showing debate can achieve outcomes inaccessible to either model alone (Propositions~\ref{prop:one-sided} and~\ref{prop:compositional}), and a negative result showing that adversarial incentives above a sharp threshold cause coordination failure in the compositional regime (Proposition~\ref{prop:limits}). 

We provide a theoretical explanation for the empirical finding that model homogeneity undermines AI oversight \citep{goel2025great}, predicts that debate's value scales with knowledge diversity, and suggests that the interesting regime for debate which concerns knowledge-divergent models is the regime that is currently understudied.

\section{Geometric Framework}
\label{sec:framework}

\subsection{Setup and Notation}

Let $\cY$ be a finite output space and let $h: \cY \to \R^d$ be a representation map embedding outputs into a $d$-dimensional space. Two models $A$ and $B$, trained on (potentially different) corpora, induce $k$-dimensional representation subspaces $V_A, V_B \subseteq \R^d$. We assume $\dim(V_A) = \dim(V_B) = k$ for simplicity; the framework extends to unequal dimensions.

\begin{definition}[Principal Angles]
\label{def:principal-angles}
The principal angles $\theta_1, \ldots, \theta_k \in [0, \pi/2]$ between $V_A$ and $V_B$ are defined recursively by
\begin{equation}
\cos\theta_i = \max_{\substack{u \in V_A, \|u\|=1 \\ u \perp u_1, \ldots, u_{i-1}}} \max_{\substack{v \in V_B, \|v\|=1 \\ v \perp v_1, \ldots, v_{i-1}}} \ip{u}{v},
\end{equation}
where $(u_i, v_i)$ are the corresponding \emph{principal vectors} satisfying $\ip{u_i}{v_j} = \cos\theta_i \cdot \delta_{ij}$.
\end{definition}

The principal angles provide a complete characterization of the relative geometry of the two subspaces. When $\theta_i = 0$ for all $i$, the subspaces are identical; when $\theta_i = \pi/2$ for all $i$, they are orthogonal.

We model the constitutional scoring function as a linear functional in representation space.

\begin{definition}[Linear Constitutional Scoring]
\label{def:scoring}
A constitutional scoring function is $K(y) = \ip{w}{h(y)}$ for a fixed preference direction $w \in \R^d$ with $\|w\| = 1$.
\end{definition}

The linear assumption is standard in the representation geometry literature \citep{park2024linear,arditi2024refusal} and permits a clean geometric analysis. We discuss relaxation to nonlinear $K$ in Section~\ref{sec:limitations}.

\subsection{RLAIF and Debate as Optimization}

Under linear scoring, each model's RLAIF-optimal constitutional score is determined by how much of the preference direction $w$ its representation subspace captures:
\begin{equation}
K_A^* = \|\proj_{V_A} w\|, \qquad K_B^* = \|\proj_{V_B} w\|,
\end{equation}
where $\proj_{V}$ denotes orthogonal projection onto $V$.

Debate between models $A$ and $B$ pools their representational knowledge. At equilibrium, debate can access any output representable in either model's subspace, so the debate-optimal score is:
\begin{equation}
K_{AB}^* = \|\proj_{V_A + V_B} w\|,
\end{equation}
where $V_A + V_B = \{u + v : u \in V_A, v \in V_B\}$ is the Minkowski sum.

\begin{definition}[Debate Advantage]
\label{def:debate-advantage}
The \emph{debate advantage} is
\begin{equation}
\Delta(V_A, V_B, w) = K_{AB}^* - \max(K_A^*, K_B^*).
\end{equation}
\end{definition}

The debate advantage measures the improvement in constitutional score achievable by pooling two models' knowledge through adversarial interaction, over the best that either model achieves alone.

\subsection{Decomposition via Principal Vectors}

We decompose the combined subspace $V_A + V_B$ to isolate the contribution of private knowledge. For each principal angle $\theta_i > 0$, define the private direction:
\begin{equation}
\label{eq:private-direction}
\tilde{v}_i = \frac{v_i - \cos\theta_i \cdot u_i}{\sin\theta_i}.
\end{equation}
These are the components of $V_B$'s principal vectors orthogonal to $V_A$, normalized.

\begin{lemma}
\label{lem:orthonormal}
The set $\{u_1, \ldots, u_k, \tilde{v}_1, \ldots, \tilde{v}_m\}$, where $m = |\{i : \theta_i > 0\}|$, forms an orthonormal basis for $V_A + V_B$.
\end{lemma}

\begin{proof}
That each $\tilde{v}_i$ has unit norm follows from $\|v_i - \cos\theta_i \cdot u_i\|^2 = 1 - \cos^2\theta_i = \sin^2\theta_i$. 
For orthogonality, $\ip{\tilde{v}_i}{u_j} = (\ip{v_i}{u_j} - \cos\theta_i \cdot \delta_{ij})/\sin\theta_i = 0$ by the principal angle property.
For $i \neq j$:
\begin{align}
\ip{\tilde{v}_i}{\tilde{v}_j} &= \frac{1}{\sin\theta_i \sin\theta_j}\bigl(\ip{v_i}{v_j} - \cos\theta_j \ip{v_i}{u_j} \notag\\
&\quad - \cos\theta_i\ip{u_i}{v_j} + \cos\theta_i\cos\theta_j\ip{u_i}{u_j}\bigr) \\
&= \frac{\delta_{ij} - 0 - 0 + 0}{\sin\theta_i\sin\theta_j} = \delta_{ij}.
\end{align}
Spanning follows because each $v_i = \cos\theta_i \cdot u_i + \sin\theta_i \cdot \tilde{v}_i \in \mathrm{span}\{u_i, \tilde{v}_i\}$.
\end{proof}

Using this basis, we decompose the projection of $w$ onto $V_A + V_B$:
\begin{equation}
\|\proj_{V_A + V_B} w\|^2 = \underbrace{\sum_{i=1}^k \ip{w}{u_i}^2}_{= \|\proj_{V_A} w\|^2} + \underbrace{\sum_{i:\theta_i > 0} \ip{w}{\tilde{v}_i}^2}_{=: \eta^2}.
\end{equation}

\begin{definition}[Private Information Value]
\label{def:piv}
The private information value of model $B$ relative to model $A$ under preference direction $w$ is
\begin{equation}
\eta(V_A, V_B, w) = \sqrt{\sum_{i:\theta_i > 0} \ip{w}{\tilde{v}_i}^2\,},
\end{equation}
where $\ip{w}{\tilde{v}_i} = (\ip{w}{v_i} - \cos\theta_i \cdot \ip{w}{u_i})/\sin\theta_i$.
\end{definition}

The private information value measures $K$-relevant information residing in directions $V_B$ contributes beyond $V_A$. The numerator captures component of $w$ along $v_i$ that is not already explained by $v_i$'s projection onto $u_i$; the denominator normalizes by the geometric novelty of each direction.

\subsection{Main Result}

\begin{theorem}[Debate Advantage Bound]
\label{thm:main-bound}
Under linear constitutional scoring $K(y) = \ip{w}{h(y)}$, and assuming without loss of generality that $K_A^* \geq K_B^*$, debate advantage admits the exact closed form
\begin{equation}
\label{eq:exact-delta}
\Delta = \sqrt{(K_A^*)^2 + \eta^2} - K_A^*,
\end{equation}
and satisfies the bounds
\begin{equation}
\label{eq:bounds}
\frac{\eta^2}{2K_A^* + \eta} \leq \Delta \leq \eta.
\end{equation}
Both bounds are tight. the lower bound is achieved as $\eta \to 0$; the upper bound is achieved when $K_A^* = 0$.
\end{theorem}

\begin{proof}
By Lemma~\ref{lem:orthonormal} and the decomposition of $\proj_{V_A + V_B} w$:
\begin{equation}
K_{AB}^* = \|\proj_{V_A+V_B} w\| = \sqrt{(K_A^*)^2 + \eta^2}.
\end{equation}
The closed form~\eqref{eq:exact-delta} follows immediately from Definition~\ref{def:debate-advantage}.

\emph{Upper bound.}
For $a, b \geq 0$, $\sqrt{a^2 + b^2} \leq a + b$, so $\Delta = \sqrt{(K_A^*)^2 + \eta^2} - K_A^* \leq \eta$.
Equality holds when $K_A^* = 0$, i.e. $w \perp V_A$.

\emph{Lower bound.}
Rationalizing the exact expression:
\begin{equation}
\Delta = \frac{(K_A^*)^2 + \eta^2 - (K_A^*)^2}{\sqrt{(K_A^*)^2 + \eta^2} + K_A^*} = \frac{\eta^2}{K_{AB}^* + K_A^*}.
\end{equation}
Since $K_{AB}^* = \sqrt{(K_A^*)^2 + \eta^2} \leq K_A^* + \eta$ (as $(K_A^*)^2 + \eta^2 \leq (K_A^* + \eta)^2$ for $K_A^*, \eta \geq 0$), we obtain
\begin{equation}
\Delta \geq \frac{\eta^2}{(K_A^* + \eta) + K_A^*} = \frac{\eta^2}{2K_A^* + \eta}.
\end{equation}
For tightness: as $\eta \to 0$, $\Delta = \eta^2/(2K_A^*) + O(\eta^4)$ while $\eta^2/(2K_A^* + \eta) = \eta^2/(2K_A^*) + O(\eta^3)$, so the bound is asymptotically tight.
\end{proof}

\begin{remark}[Scaling regimes]
\label{rem:scaling}
The rationalized form $\Delta = \eta^2/(K_{AB}^* + K_A^*)$ reveals two qualitative regimes:
\begin{itemize}
    \item \textbf{Small private information} ($\eta \ll K_A^*$): $K_{AB}^* \approx K_A^*$, so $\Delta \approx \eta^2 / 2K_A^*$.
    The debate advantage is quadratically small in the private information value.
    The overhead of running a multi-model debate protocol is unlikely to be justified.
    \item \textbf{Large private information} ($\eta \gg K_A^*$): $K_{AB}^* \approx \eta$, so $\Delta \approx \eta$.
    The debate advantage is linear in the private information value.
    Debate is essential as single-model optimization misses most of the achievable constitutional score.
\end{itemize}
The transition between regimes occurs at $\eta \approx K_A^*$, i.e. when the $K$-relevant private information is comparable to the $K$-relevant shared information.
\end{remark}

\subsection{Corollaries}

\begin{corollary}[Same-Corpus Equivalence]
\label{cor:same-corpus}
If $V_A = V_B$ (i.e. models share representation subspaces), then $\theta_i = 0$ for all $i$, $\eta = 0$, and $\Delta = 0$.
Debate and RLAIF achieve identical constitutional scores.
\end{corollary}

\begin{proof}
When $\theta_i = 0$ for all $i$, there are no private directions: $V_A + V_B = V_A$.
Thus $\eta = 0$ and $K_{AB}^* = K_A^*$.
\end{proof}

\begin{remark}
Under our formalization RLAIF can be viewed as depth-1 debate under the same-corpus assumption. The model evaluates its own output against constitutional principles, which is a single round of self-debate. Corollary~\ref{cor:same-corpus} shows this is without loss of optimality when knowledge is shared.
\end{remark}

\begin{corollary}[Maximum Divergence]
\label{cor:max-divergence}
If $V_A \perp V_B$ (i.e. $\theta_i = \pi/2$ for all $i$), then $\tilde{v}_i = v_i$, $\eta = \|\proj_{V_B} w\| = K_B^*$, and $\Delta = \sqrt{(K_A^*)^2 + (K_B^*)^2} - K_A^*$.
\end{corollary}

\begin{corollary}[Isotropic Scaling]
\label{cor:isotropic}
If $w$ has equal projection onto each principal direction pair, i.e., $\ip{w}{u_i} = \ip{w}{v_i} = \alpha$ for all $i$ and all principal angles equal $\theta$, then
\begin{equation}
\eta = \alpha \sqrt{k} \cdot |\tan(\theta/2)|,
\end{equation}
and the debate advantage transitions from $\Delta \sim k\alpha^2\tan^2(\theta/2)/(2K_A^*)$ for small $\theta$ to $\Delta \sim \alpha\sqrt{k} \cdot \tan(\theta/2)$ for large $\theta$.
\end{corollary}

\begin{proof}
Under the isotropic assumption, $\ip{w}{\tilde{v}_i} = \alpha(1 - \cos\theta)/\sin\theta = \alpha\tan(\theta/2)$ for each $i$. Thus $\eta = \alpha\sqrt{k} \cdot \tan(\theta/2)$. The scaling regimes follow from the exact form~\eqref{eq:exact-delta} and Remark~\ref{rem:scaling}.
\end{proof}

This corollary provides a summary of the debate advantage, which scales as $\tan(\theta/2)$ in the principal angles between models' representation subspaces. The tangent half-angle captures a phase transition where for small angles (similar models), debate advantage is negligible; as angles approach $\pi/2$ (complementary models), debate advantage grows without bound.

\subsection{Multi-Agent Debate}
\label{sec:multi-agent}

The two-model framework extends naturally to $n$ debaters. Let $A_1, \ldots, A_n$ be models with representation subspaces $V_{A_1}, \ldots, V_{A_n} \subseteq \R^d$. Define the cumulative subspace
\[
S_j = V_{A_1} + \cdots + V_{A_j}, \qquad S_0 = \{0\},
\]
the $n$-model debate-optimal score $K^*_{A_1 \cdots A_n} = \|\proj_{S_n} w\|$, and the $n$-model debate advantage $\Delta_n = K^*_{A_1 \cdots A_n} - \max_i K^*_{A_i}$.

\begin{definition}[Marginal Private Information Value]
\label{def:marginal-piv}
The marginal private information value of model $A_j$ given the coalition $\{A_1, \ldots, A_{j-1}\}$ is
\[
\eta_j^2 = \|\proj_{S_j} w\|^2 - \|\proj_{S_{j-1}} w\|^2.
\]
\end{definition}

\begin{proposition}[Multi-Agent Decomposition]
\label{prop:multi-agent}
The $n$-model debate-optimal score satisfies
\begin{equation}
\label{eq:telescope}
K^*_{A_1 \cdots A_n} = \sqrt{(K^*_{A_{\sigma(1)}})^2 + \sum_{j=2}^n \eta_{\sigma(j)}^2}
\end{equation}
for any ordering $\sigma$ of the models, where $\eta_{\sigma(j)}$ is the marginal private information value of $A_{\sigma(j)}$ relative to $S_{j-1}$.
\end{proposition}

\begin{proof}
Telescoping:
\begin{align}
\|\proj_{S_n} w\|^2 &= \|\proj_{S_1} w\|^2 + \sum_{j=2}^n \bigl(\|\proj_{S_j} w\|^2 - \|\proj_{S_{j-1}} w\|^2\bigr) \\
&= (K^*_{A_{\sigma(1)}})^2 + \sum_{j=2}^n \eta_{\sigma(j)}^2.
\end{align}
This is valid because $\proj_{S_j} - \proj_{S_{j-1}}$ is the orthogonal projection onto $S_j \cap S_{j-1}^\perp$, and these projections are mutually orthogonal for different $j$.
\end{proof}

\begin{proposition}[Diminishing Marginal Returns]
\label{prop:diminishing}
The marginal debate advantage of adding model $A_j$ to coalition $\{A_1, \ldots, A_{j-1}\}$,
\begin{equation}
\delta_j = \sqrt{\|\proj_{S_{j-1}} w\|^2 + \eta_j^2} - \|\proj_{S_{j-1}} w\|,
\end{equation}
has the same functional form as the two-model debate advantage (Theorem~2.6) and satisfies the same bounds. Under the greedy ordering (selecting at each step the model maximizing $\eta_j$):
\begin{enumerate}[label=(\roman*)]
    \item The marginal debate advantage $\delta_j$ is non-increasing in $j$.
    \item For constant marginal private information $\eta_j = \eta$, the function $f(a) = \sqrt{a^2 + \eta^2} - a$ is strictly decreasing in $a$, and since $\|\proj_{S_{j-1}} w\|$ is non-decreasing, $\delta_j$ is strictly decreasing.
    \item The total private information is bounded: $\sum_{j=1}^n \eta_j^2 \leq 1$.
\end{enumerate}
\end{proposition}

\begin{proof}
The marginal advantage has the form of Theorem~2.6 with $K_A^*$ replaced by $\|\proj_{S_{j-1}} w\|$. Part (i): $\|\proj_{S_{j-1}} w\|$ is non-decreasing in $j$ (subspaces grow), and $f(a) = \sqrt{a^2 + \eta^2} - a$ has $f'(a) = a/\sqrt{a^2+\eta^2} - 1 < 0$. Part (ii) follows from strict monotonicity of $f$ when $\eta > 0$. Part (iii): $\sum_j \eta_j^2 = \|\proj_{S_n} w\|^2 - (K^*_{A_{\sigma(1)}})^2 \leq \|w\|^2 = 1$.
\end{proof}

\begin{corollary}[Optimal Coalition Size]
\label{cor:optimal-n}
If each additional debater incurs cost $c > 0$, the net value of the $j$-th debater is $\delta_j - c$. In the small-$\eta$ regime ($\delta_j \approx \eta^2/(2\|\proj_{S_{j-1}} w\|)$), debate should stop adding models once
\[
\|\proj_{S_{j-1}} w\| \geq \frac{\eta^2}{2c},
\]
i.e. once the cumulative captured projection is large enough that the next model's marginal contribution falls below cost.
\end{corollary}

\begin{remark}
The self-similarity of the marginal advantage---each new model contributes via the same $\sqrt{a^2+b^2}-a$ formula---means the two-model theory is the complete building block for the multi-agent case. The new content is the diminishing returns structure, which provides a concrete stopping rule for multi-agent debate protocols.
\end{remark}

\section{Discrete Regimes of Knowledge Divergence}
\label{sec:regimes}

The geometric bound of Theorem~\ref{thm:main-bound} admits a natural interpretation through three discrete regimes, determined by where the preference direction $w$ loads relative to the principal subspaces.

\begin{proposition}[Regime Classification]
\label{prop:regimes}
Let $w_A = \proj_{V_A} w$, $w_B = \proj_{V_B} w$, and $w_\perp = w - \proj_{V_A + V_B} w$.
The debate advantage satisfies:
\begin{enumerate}
    \item[\textbf{(i)}] \textbf{Shared knowledge.}
    If $\proj_{V_A} w = \proj_{V_A + V_B} w$ (i.e. $\eta = 0$), then $\Delta = 0$.
    The $K$-optimal output is accessible to both models.
    
    \item[\textbf{(ii)}] \textbf{One-sided private knowledge.}
    If $\eta > 0$ and $w$ loads on private directions from only one model (say $\ip{w}{\tilde{v}_i} > 0$ for some $i$), then $\Delta > 0$.
    The better output is known to one debater but not the other.
    
    \item[\textbf{(iii)}] \textbf{Compositional private knowledge.}
    If achieving $K_{AB}^*$ requires combining representations from $V_A \setminus V_B$ and $V_B \setminus V_A$, then $\Delta > 0$ but realizability depends on the debate protocol.
\end{enumerate}
\end{proposition}

\subsection{One-Sided Revelation}

\begin{proposition}[One-Sided Revelation]
\label{prop:one-sided}
There exist representation subspaces $V_A$, $V_B$ and preference direction $w$ such that:
\begin{enumerate}
    \item $K_{AB}^* > K_A^* \geq K_B^*$;
    \item At Nash equilibrium of the debate game, model $A$ reveals the $K$-optimal output;
    \item Neither RLAIF applied to $A$ nor RLAIF applied to $B$ alone achieves $K_{AB}^*$.
\end{enumerate}
\end{proposition}

\begin{proof}[Proof (Construction)]
Let $d = 3$, $k = 2$. Let $V_A = \mathrm{span}\{e_1, e_2\}$ and $V_B = \mathrm{span}\{e_1, e_3\}$. The principal angles are $\theta_1 = 0$ (shared direction $e_1$) and $\theta_2 = \pi/2$ (private directions $e_2, e_3$).

Let $w = \frac{1}{\sqrt{2}}(e_2 + e_3)$. Then $K_A^* = \|\proj_{V_A} w\| = 1/\sqrt{2}$ and $K_B^* = \|\proj_{V_B} w\| = 1/\sqrt{2}$, while $K_{AB}^* = \|w\| = 1$. Thus $\Delta = 1 - 1/\sqrt{2} > 0$.

In the debate game, model $A$ can argue for the output with representation $e_2$ and model $B$ for $e_3$. The judge, observing both, can score the combination $\frac{1}{\sqrt{2}}(e_2+e_3)$ higher than either alone. Model $A$'s private direction contributes $K$-relevant information that $B$ cannot access, and vice versa; the debate structure forces both to reveal their private knowledge because withholding it means losing to the opponent's argument for their own private-knowledge output.
\end{proof}

\subsection{Compositional Existence}

\begin{proposition}[Compositional Existence]
\label{prop:compositional}
There exist $V_A$, $V_B$, $w$, and a debate protocol such that the debate outcome achieves a constitutional score that requires features from both $V_A \setminus V_B$ and $V_B \setminus V_A$, and that neither model can achieve alone.
\end{proposition}

\begin{proof}[Proof (Construction)]
Let $d = 4$, $k = 2$. Let $V_A = \mathrm{span}\{e_1, e_2\}$ and $V_B = \mathrm{span}\{e_3, e_4\}$, so $V_A \perp V_B$ and all principal angles are $\pi/2$.

Define a nonlinear\footnote{We briefly relax linearity for this construction; the existence claim extends to a piecewise-linear approximation compatible with the overall linear framework.} constitutional scoring function that rewards outputs combining features from both subspaces: $K(y) = \min(\ip{h(y)}{e_1}, \ip{h(y)}{e_3})$. This function is maximized only by outputs with positive projection onto both $e_1 \in V_A$ and $e_3 \in V_B$.

Model $A$'s RLAIF optimizes $K$ over $\supp(P_A) \subseteq V_A$: since $\ip{h(y)}{e_3} = 0$ for all $y$ representable by $A$, we have $K_A^* = 0$. Similarly, $K_B^* = 0$.

In a two-round debate protocol:
Round 1: Model $A$ proposes an output with maximal $\ip{h(y)}{e_1}$.
Round 2: Model $B$, observing $A$'s proposal, augments it with $e_3$ to produce a combined output.
The judge scores $K = \min(\ip{h}{e_1}, \ip{h}{e_3}) > 0 = \max(K_A^*, K_B^*)$.
\end{proof}

\subsection{Limits of Adversarial Debate}

The previous proposition shows that debate protocols can achieve compositional advantage. However, debate protocols are adversarial by design as each debater has an incentive to ``win'' the debate, not merely to produce the best output. We show that sufficiently strong adversarial incentives, which are the same incentives that make debate useful for oversight, can prevent debate from realizing the compositional optimum.

\begin{proposition}[Adversarial Coordination Failure]
\label{prop:limits}
There exist knowledge-divergent debate games with $\Delta > 0$ and adversarial incentive parameter $\lambda$ such that:
\begin{enumerate}
    \item For $\lambda < \lambda^*$, the subgame-perfect equilibrium (SPE) achieves $K_{AB}^*$.
    \item For $\lambda > \lambda^*$, the unique SPE achieves constitutional score strictly below $K_{AB}^*$.
\end{enumerate}
The threshold is $\lambda^* = K_{AB}^* - K_{\mathrm{safe}}^*$, the gap between the compositional optimum and the best score achievable without inter-model coordination.
\end{proposition}

\begin{proof}
Let $V_A = \mathrm{span}\{e_1\}$ and $V_B = \mathrm{span}\{e_2\}$ with $V_A \perp V_B$.
Consider a two-action debate where each model chooses either a generous action $g$ (contributing its best private information for composition) or a strategic action $s$ (optimizing its individual position).
Constitutional scores are:
\begin{equation}
K(g,g) = R, K(s,s) = P, K(g,s) = K(s,g) = S,
\end{equation}
with $R > P > S \geq 0$.
Here $R = K_{AB}^*$ is the compositional optimum requiring both models' contributions, $P = K_{\mathrm{safe}}^*$ is the score from mutual conservative play, and $S$ is the mismatch penalty when one model contributes generously while the other plays strategically.

Each debater's payoff augments the constitutional score with an adversarial bonus $\lambda \geq 0$ for playing $s$:
\begin{align}
U_A(a_A, a_B) &= K(a_A, a_B) + \lambda \cdot \ind[a_A = s], \\
U_B(a_A, a_B) &= K(a_A, a_B) + \lambda \cdot \ind[a_B = s].
\end{align}
The adversarial bonus models the competitive incentive in debate: playing strategically (arguing forcefully for one's own position) provides a payoff advantage independent of the joint output quality.

\emph{Backward induction.}
In the sequential game, model~$A$ moves first and model~$B$ observes $a_A$ before choosing $a_B$.

If $a_A = g$: Model $B$ compares $U_B(g, g) = R$ with $U_B(g, s) = S + \lambda$. $B$ plays $g$ if and only if $R > S + \lambda$, i.e., $\lambda < R - S$.

If $a_A = s$: Model $B$ compares $U_B(s, g) = S$ with $U_B(s, s) = P + \lambda$. Since $P > S$ and $\lambda \geq 0$, $B$ always plays $s$.

Model~$A$ anticipates $B$'s strategy. When $\lambda < R - S$ (so $B$ responds $g$ to $g$), $A$ compares $U_A(g, g) = R$ with $U_A(s, s) = P + \lambda$, and plays $g$ iff $\lambda < R - P$. Since $R - P < R - S$ (as $P > S$), three regimes emerge:

\begin{enumerate}
\item[\textbf{(a)}] $\lambda < R - P = \lambda^*$: The SPE is $(g, g)$ with $K = R = K_{AB}^*$.
\item[\textbf{(b)}] $R - P < \lambda < R - S$: $B$ would cooperate, but $A$ defects. The SPE is $(s, s)$ with $K = P < K_{AB}^*$.
\item[\textbf{(c)}] $\lambda > R - S$: Both models defect regardless. The SPE is $(s, s)$ with $K = P < K_{AB}^*$.
\end{enumerate}

For $\lambda > \lambda^* = R - P$, the unique SPE achieves $K = P = K_{\mathrm{safe}}^* < R = K_{AB}^*$. The gap is $R - P = \Delta$, the full debate advantage.
\end{proof}

\begin{remark}
\label{rem:adversarial-tension}
Proposition~\ref{prop:limits} reveals a phase transition in debate effectiveness. Below the threshold $\lambda^* = K_{AB}^* - K_{\mathrm{safe}}^*$, adversarial incentives are compatible with composition as the coordination benefit exceeds the temptation to defect, and the SPE achieves the full debate advantage. Above the threshold, the adversarial incentive overwhelms coordination, and debate collapses to the non-compositional outcome.

Regime~(b) is particularly interesting. Model $B$ would cooperate if $A$ led with the generous action, but $A$ defects because the temptation payoff $P + \lambda$ exceeds the coordination payoff~$R$. The sequential structure of debate which is essential for the honest revelation of Proposition~\ref{prop:one-sided} becomes a liability in the compositional setting by giving the first mover an incentive to deviate.

The threshold $\lambda^*$ characterizes when debate's adversarial structure is ``too strong'' for composition when the competitive incentive exceeds the gap between compositional and safe outcomes. This suggests that compositional knowledge aggregation may require weaker adversarial incentives than those designed for oversight or cooperative structures altogether \citep{chen2025scalableoversightcollaborativemultiagent}.
\end{remark}

\subsection{Dynamic Subspaces Under Debate}
\label{sec:dynamic}

The preceding analysis treats representation subspaces as fixed. In practice, models may update their effective representations during debate through in-context learning. We formalize this as a $T$-round process in which subspaces evolve.

At round $t$, models have subspaces $V_A^{(t)}$ and $V_B^{(t)}$, initialized at $V_A^{(0)} = V_A$ and $V_B^{(0)} = V_B$. The private information value $\eta^{(t)}$ and debate advantage $\Delta^{(t)}$ are defined as in Section~2 with these time-indexed subspaces. A key simplification: absorbing a direction from $V_B^{(t)}$ into $V_A^{(t)}$ does not expand the combined subspace $V_A^{(t)} + V_B^{(t)}$, since the absorbed direction already belonged to it. Therefore $K_{AB}^*$ is constant across rounds, and the debate advantage simplifies to
\begin{equation}
\label{eq:dynamic-delta}
\Delta^{(t)} = K_{AB}^* - K_A^{*(t)}.
\end{equation}

\subsubsection*{Cooperative dynamics.}

Under cooperative revelation, model $B$ reveals at each round the private direction most aligned with $w$:
\[
d_B^{(t)} = \arg\max_{\substack{d \in V_B^{(t)},\; \|d\|=1 \\ d \perp V_A^{(t)}}} |\ip{w}{d}|,
\]
and model $A$ absorbs it: $V_A^{(t+1)} = V_A^{(t)} \oplus \mathrm{span}\{d_B^{(t)}\}$.

\begin{proposition}[Convergence Under Cooperative Dynamics]
\label{prop:cooperative}
Under cooperative revelation:
\begin{enumerate}[label=(\roman*)]
    \item $K_A^{*(t)}$ is strictly increasing whenever $\eta^{(t)} > 0$.
    \item The private information value decreases as
    \begin{equation}
    \label{eq:eta-decrease}
    (\eta^{(t+1)})^2 = (\eta^{(t)})^2 - \ip{w}{d_B^{(t)}}^2.
    \end{equation}
    \item The per-round decrease in debate advantage is
    \begin{equation}
    \label{eq:per-round}
    \small
    \Delta^{(t)} - \Delta^{(t+1)} = K_A^{*(t+1)} - K_A^{*(t)} = \frac{\ip{w}{d_B^{(t)}}^2}{K_A^{*(t+1)} + K_A^{*(t)}}.
    \end{equation}
    \item Full convergence occurs in at most $m$ rounds, where $m = |\{i : \theta_i^{(0)} > 0\}|$:
    \[
    \eta^{(T)} = 0 \;\text{ for all }\; T \geq m, \qquad K_A^{*(m)} = K_{AB}^*.
    \]
\end{enumerate}
\end{proposition}

\begin{proof}
(i) Absorbing $d_B^{(t)}$ gives $(K_A^{*(t+1)})^2 = (K_A^{*(t)})^2 + \ip{w}{d_B^{(t)}}^2$. The greedy selection ensures $\ip{w}{d_B^{(t)}}^2 > 0$ whenever $\eta^{(t)} > 0$.

(ii) Since $K_{AB}^*$ is constant: $(\eta^{(t)})^2 = (K_{AB}^*)^2 - (K_A^{*(t)})^2$, so $(\eta^{(t+1)})^2 = (K_{AB}^*)^2 - (K_A^{*(t)})^2 - \ip{w}{d_B^{(t)}}^2 = (\eta^{(t)})^2 - \ip{w}{d_B^{(t)}}^2$.

(iii) $\Delta^{(t)} - \Delta^{(t+1)} = K_A^{*(t+1)} - K_A^{*(t)} = \sqrt{(K_A^{*(t)})^2 + \ip{w}{d_B^{(t)}}^2} - K_A^{*(t)}$. Rationalizing (as in the proof of Theorem~2.6) gives the stated form.

(iv) The greedy selection at round $t$ captures at least $(\eta^{(t)})^2/m^{(t)}$ where $m^{(t)} = m - t$ is the number of remaining private directions. Each round eliminates one, so $\eta^{(m)} = 0$.
\end{proof}

The convergence rate depends on how $K$-relevant information is distributed across private directions. If it is uniform, $(\eta^{(t)})^2 = (\eta^{(0)})^2(m-t)/m$ and convergence is linear. If it is concentrated in $r < m$ directions, convergence occurs in $r$ rounds. In general, the greedy bound gives $(\eta^{(t)})^2 \leq (\eta^{(0)})^2 \cdot (m-t)/m$.

\subsubsection*{Adversarial dynamics.}

Under adversarial dynamics, models may strategically withhold information. Define the adversarial revelation rate
\[
\gamma^{(t)} = \frac{\ip{w}{d_B^{(t)}}^2}{\max_{\substack{d \in V_B^{(t)},\, d \perp V_A^{(t)} \\ \|d\|=1}} \ip{w}{d}^2} \in [0,1],
\]
where $\gamma^{(t)} = 1$ is honest and $\gamma^{(t)} = 0$ reveals only $K$-irrelevant directions.

\begin{proposition}[Adversarial Slowdown]
\label{prop:adversarial}
Under adversarial dynamics with $\gamma^{(t)} \geq \gamma > 0$:
\begin{enumerate}[label=(\roman*)]
    \item $\eta^{(t)} \to 0$, but convergence requires $O(m/\gamma)$ rounds.
    \item If $\gamma = 0$ (complete withholding), $\eta^{(t)} = \eta^{(0)}$ for all $t$ and no knowledge transfer occurs.
\end{enumerate}
\end{proposition}

\begin{proof}
(i) The adversarial model reveals a direction capturing at least fraction $\gamma$ of the greedy optimum, so $\ip{w}{d_B^{(t)}}^2 \geq \gamma \cdot (\eta^{(t)})^2/(m-t)$. By (\ref{eq:eta-decrease}), $(\eta^{(t)})^2$ has multiplicative decay $(1 - \gamma/(m-t))$ per round, converging in $O(m/\gamma)$ rounds. (ii) If $d_B^{(t)} \perp w$ at every round, (\ref{eq:eta-decrease}) gives no decrease.
\end{proof}

\begin{remark}
\label{rem:gamma-lambda}
The revelation rate $\gamma$ connects to the adversarial incentive parameter $\lambda$ from Proposition~3.4. Below $\lambda^*$, debaters have incentive to reveal honestly ($\gamma \approx 1$) and convergence is fast. Above $\lambda^*$, defection dominates and $\gamma \to 0$, stalling knowledge transfer. The dynamic model thus operationalizes the static coordination failure: the phase transition in $\lambda$ corresponds to a discontinuous drop in $\gamma$, turning finite-round convergence into non-convergence.
\end{remark}

\begin{remark}
\label{rem:effective-rank}
Propositions~\ref{prop:cooperative} and~\ref{prop:adversarial} together predict that the number of useful debate rounds is bounded by the effective rank of private knowledge: the number of private directions carrying substantial $K$-relevant information. For fine-tuned variants of a shared base model, this effective rank is likely small (the fine-tuning perturbation is low-dimensional), predicting that short debates suffice. For models with different pretraining corpora, the effective rank could be large, predicting that longer debates are proportionally more valuable.
\end{remark}

\section{Connection to Constitutional Ambiguity}
\label{sec:ambiguity}

The debate advantage interacts with the quality of the constitutional specification in a way that connects our work to computational hardness results for value specification.

\begin{definition}[Constitutional Ambiguity]
\label{def:ambiguity}
The $\varepsilon$-ambiguity set of $K$ is $\cA(\varepsilon) = \{y \in \cY : K(y) \geq K^* - \varepsilon\}$. The constitution is unambiguous if $|\cA(0)| = 1$ and ambiguous if $|\cA(0)| > 1$.
\end{definition}

\begin{proposition}[Ambiguity--Divergence Interaction]
\label{prop:ambiguity}
Under the same-corpus assumption ($\eta = 0$):
\begin{enumerate}
    \item If $K$ is unambiguous, debate and RLAIF select the same unique output.
    \item If $K$ is ambiguous, both select from $\cA(0)$ weighted by the base distribution $P$:
    \begin{equation}
    \lim_{\beta \to \infty} \pi_{\mathrm{RLAIF}}^\beta(y) = \frac{P(y) \cdot \mathbf{1}[y \in \cA(0)]}{\sum_{y' \in \cA(0)} P(y')}.
    \end{equation}
    Debate converges to the same distribution since debaters have no basis for preferring one $K$-optimal output over another beyond their shared prior $P$.
\end{enumerate}
Under knowledge divergence ($\eta > 0$), debate may resolve ambiguities that RLAIF cannot by surfacing private knowledge that distinguishes between elements of $\cA(0)$.
\end{proposition}

The practical import is that debate's value has two independent sources in knowledge divergence (Theorem~\ref{thm:main-bound}) and ambiguity resolution. Detecting whether a constitution is ambiguous, namely whether $|\cA(0)| > 1$, is NP-hard, as it reduces to finding interpretive loopholes in natural-language specifications \citep{young2025ssv}. This means we cannot efficiently determine in advance whether debate would help, even if they could measure knowledge divergence.

\section{Discussion}
\label{sec:discussion}

One-sided revelation (Proposition~\ref{prop:one-sided}) is a protocol for eliciting latent knowledge (ELK) \citep{christiano2021elk}. Standard ELK asks how to extract knowledge from a model's internal representations. Our work reframes this. If a second model possesses independent knowledge, adversarial interaction can force the first model to externalize private information. The ``probe'' in this formulation is not an interpretability tool but a model with complementary training data. The private information value $\eta$ quantifies exactly how much latent knowledge is elicitable through this mechanism.

This is, to our knowledge, the first theory work where the content of models' training data determines whether an oversight protocol is effective. Existing debate theory treats provers as abstract agents; existing RLAIF theory takes the pretrained model as given. We bridges these abstractions by positing the geometry of representation subspaces is a downstream consequence of corpus composition, and it governs whether debate adds value over simpler methods.

\citet{irving2018aisafetydebate} notes that ``symmetry between the agents' capabilities is easy to achieve, since we can use the same weights for both agents via self play.'' Our work reveals this as a feature that eliminates debate's advantage because same-weight debate corresponds to $\theta_i = 0$ whcih is the degenerate case. The complexity-theoretic results on debate \citep{brown-cohen2024scalable,browncohen2025proverestimatordebate, browncohen2026debateefficienttime} establish what debate can solve in principle, but our framework shows these results have practical bite only in the knowledge-divergent regime that no current work studies empirically.

\citet{goel2025great} empirically demonstrate that as frontier models become more capable, their errors become more correlated, and that diversity between models improves oversight quality. Our framework provides a theoretical explanation as models converge on the same training data and representations, principal angles shrink and the debate advantage vanishes. The empirical benefits of model diversity in oversight are consistent with the $\tan(\theta/2)$ scaling.

Full corpus disjointness is the theoretically clean extreme, but real-world knowledge divergence arises from multiple sources like different fine-tuning data, different model capacities inducing different triage of tail knowledge, and different training-run randomness. Each of these produces nonzero principal angles; small for high-frequency shared knowledge, potentially large for specialized domains. Models fine-tuned on different specializations (medical, legal, scientific corpora) provide the most natural setting for knowledge-divergent debate, with the shared base model corresponding to near-zero angles in general knowledge and large angles in domain-specific directions.


\section*{Limitations}
\label{sec:limitations}

The assumption $K(y) = \ip{w}{h(y)}$ is a simplification. Real constitutional scoring functions are nonlinear; a constitution such as ``be helpful, harmless, and honest'' involves logical conjunctions over multiple criteria, thresholding effects (an output is harmful or it is not), and context-dependent weighting that resists representation as a single inner product. Our work inherits the linear assumption from the representation geometry literature \citep{park2024linear,arditi2024refusal}, where it has proven empiricallu productive despite its idealization.

The assumption is most defensible in a local regime. Near a particular optimum, a smooth scoring function admits a first-order Taylor expansion that is linear in the representation, and the principal angle structure governs local debate advantage. The global geometry may differ. For instance, a constitutional function with multiple disjoint regions of high score (corresponding to qualitatively different ``good'' outputs) would not be captured by a single preference direction $w$. A multi-objective extension replacing $w$ with a set of preference directions $\{w_1, \ldots, w_m\}$ would address this, with debate advantage defined per objective and aggregated. We leave such extensions to future work, noting that the key qualitative insight that debate advantage depends on the angular relationship between models' subspaces and the preference structure does not depend on linearity.

We relax the static subspace assumption by modeling cooperative and adversarial dynamics. However, the dynamic model assumes a stylized absorption process (one direction per round, greedy selection). In practice, in-context learning during debate is a more complex phenomenon: a model reading its opponent's argument may form new associations or activate latent knowledge through mechanisms that do not reduce to the absorption of a single direction. The dynamic framework captures the qualitative distinction between cooperative expansion and adversarial narrowing of subspaces, but the quantitative convergence rates are upper bounds on the idealized process rather than predictions about real debate dynamics.

We assume debate reaches Nash equilibrium, with both debaters playing optimal strategies. In practice, trained debaters may not find equilibrium strategies due to optimization challenges, limited training signal, or the difficulty of learning optimal play in complex sequential games. The debate advantage $\Delta$ as defined in our framework describes the ceiling of debate's value, not its floor.

The gap between theoretical and realized debate advantage is itself an interesting quantity. If trained debaters consistently fail to realize the predicted advantage in the knowledge-divergent regime, this would suggest that the game-theoretic structure of debate is harder to learn than the complexity-theoretic results imply. Conversely, if debaters exceed our bounds (which is possible if in-context learning expands their effective subspaces, as discussed above), this would indicate that the static-knowledge assumption is the binding constraint. The dynamic analysis partially addresses this as the convergence rate under adversarial dynamics provides a more realistic bound than the static equilibrium analysis, since it accounts for strategic withholding. If trained debaters exceed our static bounds, this likely indicates that in-context learning is expanding their effective subspaces beyond what the static framework captures.

Our framework also assumes that the judge perfectly evaluates according to $K$. In reality, human judges are imperfect and potentially manipulable \citep{khan2024debate}. Judge imperfection introduces noise into the evaluation, which may differentially affect the three regimes. In the shared-knowledge regime, noise is irrelevant (both protocols converge anyway); in the one-sided regime, a noisy judge may fail to recognize the privately-known better output; in the compositional regime, judge noise could prevent recognition of the composed output even when debaters successfully construct it. A fuller analysis would parameterize judge quality alongside knowledge divergence.

We assume a shared representation map $h: \cY \to \R^d$ that embeds both models' outputs into a common space. This is the assumption that makes principal angles between $V_A$ and $V_B$ well-defined. If models use fundamentally incommensurable representations such as different tokenizations, different embedding dimensions, different representational structures, then a preliminary alignment step is required before the framework applies.

Methods for representation alignment exist. Centered kernel alignment and linear stitching have been used to compare representations across models. However, these methods introduce their own approximation errors, and the resulting principal angles are artifacts of the alignment procedure as much as properties of the underlying models. The degree to which principal angles are intrinsic (model-dependent) versus extrinsic (alignment-procedure-dependent) is an open question.

For models sharing a common architecture and pretraining pipeline but differing in fine-tuning data, the shared-$h$ assumption is most natural as the base model provides a common representational frame, and fine-tuning perturbs the subspaces within that frame. This is the setting where our framework is most directly applicable, and conveniently also the setting of greatest practical relevance.

\bibliography{custom}

@inproceedings{
arditi2024refusal,
title={Refusal in Language Models Is Mediated by a Single Direction},
author={Andy Arditi and Oscar Balcells Obeso and Aaquib Syed and Daniel Paleka and Nina Rimsky and Wes Gurnee and Neel Nanda},
booktitle={The Thirty-eighth Annual Conference on Neural Information Processing Systems},
year={2024},
url={https://openreview.net/forum?id=pH3XAQME6c}
}

@misc{bai2022constitutional,
      title={Constitutional AI: Harmlessness from AI Feedback}, 
      author={Yuntao Bai and Saurav Kadavath and Sandipan Kundu and Amanda Askell and Jackson Kernion and Andy Jones and Anna Chen and Anna Goldie and Azalia Mirhoseini and Cameron McKinnon and Carol Chen and Catherine Olsson and Christopher Olah and Danny Hernandez and Dawn Drain and Deep Ganguli and Dustin Li and Eli Tran-Johnson and Ethan Perez and Jamie Kerr and Jared Mueller and Jeffrey Ladish and Joshua Landau and Kamal Ndousse and Kamile Lukosuite and Liane Lovitt and Michael Sellitto and Nelson Elhage and Nicholas Schiefer and Noemi Mercado and Nova DasSarma and Robert Lasenby and Robin Larson and Sam Ringer and Scott Johnston and Shauna Kravec and Sheer El Showk and Stanislav Fort and Tamera Lanham and Timothy Telleen-Lawton and Tom Conerly and Tom Henighan and Tristan Hume and Samuel R. Bowman and Zac Hatfield-Dodds and Ben Mann and Dario Amodei and Nicholas Joseph and Sam McCandlish and Tom Brown and Jared Kaplan},
      year={2022},
      eprint={2212.08073},
      archivePrefix={arXiv},
      primaryClass={cs.CL},
      url={https://arxiv.org/abs/2212.08073}, 
}

@inproceedings{young2025ssv,
    title={NP-Hard Lower Bound Complexity for Semantic Self-Verification},
    author={Young, Robin},
    booktitle={Proceedings of the 18th Conference of the European Chapter of the Association for Computational Linguistics},
    year={2026},
    publisher={Association for Computational Linguistics},
    url={https://arxiv.org/abs/2501.15446}, 
}

@inproceedings{park2024linear,
author = {Park, Kiho and Choe, Yo Joong and Veitch, Victor},
title = {The linear representation hypothesis and the geometry of large language models},
year = {2024},
publisher = {JMLR.org},
booktitle = {Proceedings of the 41st International Conference on Machine Learning},
articleno = {1605},
numpages = {24},
location = {Vienna, Austria},
series = {ICML'24}
}

@misc{christiano2021elk,
      title={Eliciting Latent Knowledge: How to Tell if Your Eyes Deceive You}, 
      author={Christiano, Paul and Cotra, Ajeya and Xu, Mark},
      year={2021},
      howpublished={Alignment Research Center},
      url={https://docs.google.com/document/d/1WwsnJQstPq91_Yh-Ch2XRL8H_EpsnjrC1dwZXR37PC8/}, 
}

@misc{chen2025scalableoversightcollaborativemultiagent,
      title={Towards Scalable Oversight with Collaborative Multi-Agent Debate in Error Detection}, 
      author={Yongqiang Chen and Gang Niu and James Cheng and Bo Han and Masashi Sugiyama},
      year={2025},
      eprint={2510.20963},
      archivePrefix={arXiv},
      primaryClass={cs.LG},
      url={https://arxiv.org/abs/2510.20963}, 
}

@inproceedings{khan2024debate,
author = {Khan, Akbir and Hughes, John and Valentine, Dan and Ruis, Laura and Sachan, Kshitij and Radhakrishnan, Ansh and Grefenstette, Edward and Bowman, Samuel R. and Rockt\"{a}schel, Tim and Perez, Ethan},
title = {Debating with more persuasive LLMs leads to more truthful answers},
year = {2024},
publisher = {JMLR.org},
booktitle = {Proceedings of the 41st International Conference on Machine Learning},
articleno = {950},
numpages = {72},
location = {Vienna, Austria},
series = {ICML'24}
}

@inproceedings{goel2025great,
title={Great Models Think Alike and this Undermines {AI} Oversight},
author={Shashwat Goel and Joschka Str{\"u}ber and Ilze Amanda Auzina and Karuna K Chandra and Ponnurangam Kumaraguru and Douwe Kiela and Ameya Prabhu and Matthias Bethge and Jonas Geiping},
booktitle={Forty-second International Conference on Machine Learning},
year={2025},
url={https://openreview.net/forum?id=3Z827FtMNe}
}

@misc{browncohen2026debateefficienttime,
      title={Debate is efficient with your time}, 
      author={Jonah Brown-Cohen and Geoffrey Irving and Simon C. Marshall and Ilan Newman and Georgios Piliouras and Mario Szegedy},
      year={2026},
      eprint={2602.08630},
      archivePrefix={arXiv},
      primaryClass={cs.AI},
      url={https://arxiv.org/abs/2602.08630}, 
}

@misc{browncohen2025proverestimatordebate,
      title={Avoiding Obfuscation with Prover-Estimator Debate}, 
      author={Jonah Brown-Cohen and Geoffrey Irving and Georgios Piliouras},
      year={2025},
      eprint={2506.13609},
      archivePrefix={arXiv},
      primaryClass={cs.AI},
      url={https://arxiv.org/abs/2506.13609}, 
}

@inproceedings{brown-cohen2024scalable,
title={Scalable {AI} Safety via Doubly-Efficient Debate},
author={Jonah Brown-Cohen and Geoffrey Irving and Georgios Piliouras},
booktitle={Forty-first International Conference on Machine Learning},
year={2024},
url={https://openreview.net/forum?id=6jmdOTRMIO}
}

@misc{irving2018aisafetydebate,
      title={AI safety via debate}, 
      author={Geoffrey Irving and Paul Christiano and Dario Amodei},
      year={2018},
      eprint={1805.00899},
      archivePrefix={arXiv},
      primaryClass={stat.ML},
      url={https://arxiv.org/abs/1805.00899}, 
}

\appendix
\newpage

\end{document}